\documentclass[10pt,twocolumn,letterpaper]{article}
\usepackage{enumitem}
\usepackage{cvpr}
\usepackage{times}
\usepackage{epsfig}
\usepackage{graphicx}
\usepackage{amsmath}
\usepackage{amssymb}
\usepackage{color}

\newcommand{\Eb}{\mathbb{E}}
\usepackage{amsmath,amsfonts,amssymb,amsthm}
\newcommand{\RR}{\mathbb{R}}
\newcommand{\nn}{\nonumber}
\newcommand{\g}{{G_u}}
\newcommand{\vb}{\mathbf{v}}
\newcommand{\di}{{D_v}}
\newcommand{\U}{\mathcal{U}}
\newcommand{\D}{\mathcal{D}}
\newcommand{\V}{\mathcal{V}}
\newcommand{\E}{\mathbb{E}}
\newcommand{\xt}{{\tilde{x}}}
\newcommand{\M}{\mathcal{M}}
\newcommand{\J}{\mathbf{J}}
\newcommand{\Y}{\mathcal{Y}}
\newtheorem{as}{Assumption}
\newtheorem{thm}{Theorem}
\newtheorem*{thm*}{Theorem}
\newtheorem{definition}{Definition}

\usepackage[pagebackref=true,breaklinks=true,letterpaper=true,colorlinks,bookmarks=false]{hyperref}

 \cvprfinalcopy 

\def\cvprPaperID{5744} 

\ifcvprfinal\fi
\begin{document}

\title{MR-GAN: Manifold Regularized Generative Adversarial Networks}

\author{Qunwei Li\\
Syracuse University\\
{\tt\small qli33@syr.edu}
\and
Bhavya Kailkhura \\
Lawrence Livermore National Laboratory\\
{\tt\small kailkhura1@llnl.gov}
\and
Rushil Anirudh \\
Lawrence Livermore National Laboratory\\
{\tt\small anirudh1@llnl.gov}
\and
Yi Zhou \\
Duke University\\
{\tt\small yi.zhou610@duke.edu}
\and
Yingbin Liang \\
The Ohio State University\\
{\tt\small liang.889@osu.edu}
\and
Pramod K. Varshney \\
Syracuse University\\
{\tt\small varshney@syr.edu}
}

\maketitle

\begin{abstract}
Despite the growing interest in generative adversarial networks (GANs), training GANs remains a challenging problem, both from a theoretical and a practical standpoint. To address this challenge, in this paper, we propose a novel way to exploit the unique geometry of the real data, especially the manifold information. More specifically, we design a method to regularize GAN training by adding an additional regularization term referred to as manifold regularizer. The manifold regularizer forces the generator to respect the unique geometry of the real data manifold and generate high quality data. Furthermore, we theoretically prove that the addition of this regularization term in any class of GANs including DCGAN and Wasserstein GAN leads to improved performance in terms of generalization, existence of equilibrium, and stability. Preliminary experiments show that the proposed manifold regularization helps in avoiding mode collapse and leads to stable training.
\end{abstract}

\section{Introduction}
Ever since the introduction of generative adversarial networks (GANs) \cite{goodfellow2014generative}, they have become one of the most widely used deep generative models to synthetically generate complicated real-world data. The core of the training of GANs is a min-max game in
which two neural networks (generator and discriminator) compete with each other: the generator tries to trick the discriminator/classifier into classifying its generated synthetic/fake data as true. 
Various applications have benefited from the utilization of GANs, e.g., video prediction, object generation, photo super resolution~(see \cite{ledig2017photo,tulyakov2017mocogan} and the references therein).

Despite the interest that GANs have drawn, the task of training GANs remains a challenging problem, both from a theoretical and a practical standpoint. 
{{Specifically, GAN training suffers from the following major problems: $a)$ \textit{mode-collapse}: the generator collapses with results in a poor generalization, i.e., producing limited varieties of samples, $b)$ \textit{lack of equilibrium}: the min-max game may not have any  equilibrium, and $c)$ \textit{instability}: even when the equilibrium exists, model parameters may oscillate, destabilize and never converge to an equilibrium.}} These failure modes result in generation of poor quality data. 

It was shown in~\cite{arjovsky2017towards} that the real data lies in a submanifold, and the generated data and real data lying in disjoint manifolds is one of the reasons for the aforementioned problems in the training of GANs. Motivated from this insight, this paper takes some initial steps towards designing GAN architectures which can exploit the unique geometry of the real data (especially the manifold information) to overcome the aforementioned problems. 
The basic idea
is simple yet powerful: in addition to the gradient
information provided by the discriminator, we want
the generator to exploit other geometric information present in the real data, such
as the manifold information. Taking advantage of this additional information, we will have more stable gradients while
training our generator. 
Specifically, we propose a novel method for incorporating geometry and regularizing the GAN training by adding an additional regularization term (referred to as manifold regularizer) with generator updates. The proposed manifold regularizer forces the generator to respect the unique geometry of the real data manifold. We theoretically prove that the addition of this regularization term in any class of GANs (including DCGAN and Wasserstein GAN) leads to improved performance in terms of generalization, equilibrium, and stability.
In practice, the manifold regularized GANs (MR-GANs) are simple to implement, and results in better performance compared to their unregularized counterparts.

\subsection{Related Work}
In the literature, not much theory exists that explains the unstable behaviour of GAN training except for \cite{arjovsky2017towards} that stands out as one of the most successful work. The authors provided important insights into mode collapse and instability in GAN training. They showed that these issues arise when the supports of the generated distribution and the true distribution are disjoint. The authors in~\cite{arora2017generalization}, on the other hand, explored questions relating to the sample complexity and expressiveness of the GAN architecture and their relation to the existence of an equilibrium. Given that an equilibrium exists, the convergence of GAN with update procedure using gradient descent was studied in~\cite{nagarajan2017gradient}.

From a practical perspective, various architectures and training objectives have been proposed to address GAN training challenges~\cite{arjovsky2017wasserstein,poole2016improved,khayatkhoei2018disconnected}. Several optimization heuristics and architectures have also been proposed to address challenges such as mode collapse~\cite{salimans2016improved,unrolled,radford2015unsupervised,mode_reg}. Methods for regularizing the discriminator for better stability were devised in~\cite{roth,numerics,nagarajan2017gradient,lecouat2018semi}. The authors in~\cite{roth} presented a stabilizing regularizer that is based on a gradient norm, where the gradient is calculated with respect to the data samples. On the other hand, the authors of \cite{numerics,nagarajan2017gradient} designed regularizers based on the norm of a gradient calculated with respect to the parameters. The authors in \cite{lecouat2018semi} applied a Jacobian regularizer to the discriminator of a feature-matching GAN to improve the performance of GAN-based semi-supervised learning. In contrast to regularizing the discriminator, this paper proposes to regularize the generator for improving GAN training. Finally, the authors in~\cite{park2017mmgan} proposed replacing the original GAN loss with a different loss function matching the statistical mean and radius of the spheres approximating the geometry of the real data and generated data. However, characterizing the geometric information of the data only by the mean and radius of loses a significant amount of geometrical information. The construction in~\cite{park2017mmgan} was purely heuristic and did not have any theoretical backing. On the contrary, we cirectly exploit the undistorted manifold information for regularizing the training of the generator rather than treating it  as a loss function and theoretically prove that the proposed approach yields improved performance in terms of generalization, existence of equilibrium, and stability.

\subsection{Contribution}
The main contributions of the paper are summarized as follows:
\begin{itemize}
    \item We propose a novel method for regularizing GAN training by incorporating an additional regularization term to respect the unique geometry of the real data manifold.
    
	\item It is proved that the proposed training objective function can be realized with a sufficiently small bias using deep neural networks (DNNs) and, thus, {{yields excellent generalization performance}}.
	
	\item It is proved that the equilibrium of the min-max game for the proposed MR-GANs exists, and can be attained by DNNs in practice.
	
	\item It is proved that the training of the proposed MR-GAN is exponentially stable around the equilibrium.

	\item It is shown empirically that the MR-GANs are able to avoid model collapse and significantly outperform several widely used baseline GAN architectures.

\end{itemize}

\section{Preliminaries}
In this section, we give a brief introduction of GANs and manifold learning. We will also briefly discuss how manifold learning principles can be exploited to have a better GAN formulation. 

\subsection{Introduction of GANs}

Throughout the paper, we use $d$ for the dimension of samples, $p$ for the number of parameters in generator/discriminator, and $m$ for the number of samples. Let $\{G_u, u \in \mathcal{U}\}$, with $\mathcal{U} \in \RR^p$, denote the class of generators, where $\g$ is a function -- which is often a neural network in practice -- from $\RR^l \rightarrow \RR^d$ indexed by $u$ that denotes the parameters of the generators. Here $\U$ denotes the possible ranges of the parameters and without loss of generality we assume that $\U$ is a subset of the unit ball. The generator $\g$ defines a distribution $\D_\g$ as follows: generate $h$ from an $l$-dimensional spherical Gaussian distribution and apply $\g$ on $h$ and generate a sample $x = G_u(h)$ from the distribution $\D_\g$. We drop the subscript $u$ in $\D_\g$ when it is clear from the context. 
Let $\{\di, v \in \mathcal{V}\}$ denote the class of discriminators, where $\di$ is function from $\RR^d$ to $[0, 1]$ and $v$ is the parameter of $\di$. Training the discriminator consists of making its output a high value (preferably $1$) when $x$ is sampled from the distribution $\D_{real}$ and a low value (preferably $0$) when $x$ is sampled from the synthetic distribution $\D_\g$. On the contrary, training the generator consists of making its synthetic
distribution ``similar'' to $\D_{real}$ in the sense that the discriminator's output tends to indicate that the two distributions are close.

The original GAN training problem \cite{goodfellow2014generative} is formulated as the following min-max game between the generator and the discriminator:
\begin{align}
\min_{u\in \U} \max_{v\in \V} \mathop{\E}_{x\sim \D_{real}}[\log \di(x)]+\mathop{\E}_{y\sim \D_{\g}} [\log(1-\di(y))].\nonumber
\end{align}

Intuitively, this says that the discriminator $\di$ should give high values $\di(x)$ to the real samples and low values $\di(y)$ to the generated examples. The log function is used because of its interpretation as the likelihood. However, in practice it can cause problems since $\log x \rightarrow -\infty$ as $x \rightarrow 0$. The objective still makes intuitive sense if we replace log by any monotone function $\phi : [0, 1] \rightarrow \RR$, which yields the
objective:
\begin{align}
\min_{u\in \U} \max_{v\in \V} \mathop{\E}_{x\sim \D_{real}}\mathop{\E}_{y\sim \D_{\g}}[\phi (\di(x))+ \phi(1-\di(y))].\nonumber
\end{align}
We call the function $\phi$ the measuring function. It should be concave so that when $\D_{real}$ and $\D_{\g}$ are the same distributions, the best strategy for the discriminator is simply to output $1/2$ and the optimal value is $2\phi(1/2)$. In later proofs, we require $\phi$ to be bounded and Lipschitz. In practice, training often uses $\phi(x) = \log(\delta + (1 - \delta)x)$ (which takes values in $[\log \delta, 0]$ and is $1/\delta$-Lipschitz)
and the recently proposed Wasserstein GAN \cite{arjovsky2017wasserstein} objective function uses $\phi(x) = x$ {{(which takes values in $[0,1]$ (by definition) and is $1$-Lipschitz)}}.

\subsection{Manifold Learning}
In several machine learning applications, the data lies on or close to the surface of one or more low-dimensional manifolds embedded in the high-dimensional ambient space. 
Attempting to uncover this manifold structure in a data set is referred to as manifold learning~\cite{belkin2006manifold}. 

Given a set $\mathbf{X}$ of data (or feature) vectors, a graph $\mathcal{G}=\{\mathbf{X},\mathbf{\Omega}\}$
is used to characterize the manifold-based relationships among these vectors. Here, $\mathbf{\Omega}=[w_{ij}]$
is  a  matrix  containing  the
weights over edges connecting graph nodes and is referred to as
the affinity matrix. The weight, $w_{ij}$, on an edge connecting two
nodes,
$\mathbf{x}_i$
and
$\mathbf{x}_j$, provides a measure of closeness between them. These weights govern various characteristics
of a graph, including structure, connectivity and compactness.
Graph-based relationships are usually characterized using
the Euclidean distance based Gaussian heat kernel given by

\begin{align}\label{weight}
w_{ij}=\left\{\begin{matrix}
\exp\left ( \frac{-\|\mathbf{x}_{i}-\mathbf{x}_{j}\|^2}{\rho} \right ) & ; & e(\mathbf{x}_{i},\mathbf{x}_{j})=1\\ 
0& ; & \text{otherwise}
\end{matrix}\right.
\end{align}
where $\rho$ is  the kernel  scale parameter and the function $e(\mathbf{x}_{i},\mathbf{x}_{j})$ indicates whether $\mathbf{x}_{i}$ lies near the neighborhood of $\mathbf{x}_{j}$. 
As an example, given input $\mathcal{G}=\{\mathbf{X},\mathbf{\Omega}\}$, manifold learning inspired learning approaches attempt to constrain the output, $\mathbf{y}=f(\mathbf{X})$\footnote{Here $f(\cdot)$ is a task specific function.}, to preserve the structure (compactness) in $\mathbf{X}$ (defined by the affinity weight matrix $\mathbf{\Omega}$). This is usually achieved by employing a regularization term along with a task specific loss function.
A case of particular recent interest in manifold regularized learning is when the support of the data is a compact submanifold $\mathcal{M}\in R^d$. In that case, one natural choice for the regularizer is $\int_{x\in \mathcal{M}}\|\nabla _{\mathcal{M}} \mathbf{y}\|^2dP_X(x)$, where $\nabla _{\mathcal{M}}$ is the gradient of $\mathbf y$ along the manifold $\mathcal{M}$ and the integral is taken over the marginal distribution. In most applications, the marginal $P_X$ is not known. Therefore, we need to get empirical estimates
of $P_X$ and the regularizer. The term
$\int_{x\in \mathcal{M}}\|\nabla _{\mathcal{M}} \mathbf{y}\|^2dP_X(x)$ may be approximated on the basis of data samples
using the graph Laplacian  $L$ associated to the data, which yields an estimate $\frac{1}{m^2}Tr(\mathbf y^TL\mathbf y)$, where $\mathbf{y}=[\mathbf{y}_1, \mathbf{y}_2,\ldots, \mathbf{y}_m]$.
For weighted graph Laplacian (or affinity matrix $\mathbf{\Omega}$), this estimate becomes $\sum \limits _{i,j} \|\mathbf{y}_{i}-\mathbf{y}_{j}\|^2w_{ij}$.

\section{Manifold Regularized GAN (MR-GAN)}
\subsection{Geometry-Aware GANs}
{{Although a reasonable approach for GANs would be to use the conventional manifold regularizer $Tr(\mathbf{y}^TL\mathbf{y})$ at the generator to force the generated data $\mathbf{y}$ to respect the geometry of the real data $\mathbf{X}$, our initial experiments suggested that the conventional regularizer does not perform well in practice. In fact, with the conventional manifold regularizer 
		at the generator, our theoretical analysis also indicates that the equilibrium cannot be guaranteed. Therefore, we propose a novel regularizer at the generator to force the generated data to respect the geometry of the real data. Furthermore, for this new formulation, we theoretically show that some of the issues with GAN training can be overcome.}}
\subsection{Proposed GAN Architecture}
Motivated by the considerations above, in this section we propose a novel regularization penalty for the
generator updates, which employs a term based on the gradient of the embedding function $\psi$ in the intrinsic manifold, to incorporate the fact that the real data is indeed extremely concentrated on a low dimensional manifold \cite{narayanan2010sample}. The embedding function $\psi$ serves two purposes. First, it extracts useful information from the raw data for better inference. Second, it is a dimension-reduction mapping, which can prevent overfitting during training. {As we will show later that the regularization term does not change the parameter values at the equilibrium point, and further enhances the local stability of the optimization procedure.}  Specifically, we propose the following regularized objective of MR-GAN\footnote{Please see \eqref{mr-emp} for an empirical version of the MR-GAN formulation.} as follows:
\begin{align}\label{objective_function}
\min_{u\in \U} \max_{v\in \V} &\mathop{\E}_{x\sim \D_{real}}\mathop{\E}_{y\sim \D_{\g}}[\phi (\di(x))+ \phi(1-\di(y))
\nn\\&+\lambda \int_{x \sim\M} \|\nabla_\M (\psi(y)-\psi(x))\|^2dPx],
\end{align}
where $\psi$ is an embedding function which takes the form $\psi: x \rightarrow \xt$, and $\xt$ lies within a manifold embedded in $\RR^d$. 
 Essentially, the regularizer is the squared magnitude of the gradient  of the embedding function in the intrinsic manifold, with respect to the difference between the real and generated data. {When the support of distribution $\D_{real}$ lies in the manifold $\M$, the objective \eqref{objective_function} becomes the following because we have an expectation operator over the distribution of the real data:
	\begin{align}\label{problem}
	\min_{u\in \U} \max_{v\in \V} \mathop{\E}_{x\sim \D_{real}}\mathop{\E}_{y\sim \D_{\g}}&[\phi (\di(x))+ \phi(1-\di(y))+\nn\\&
	\lambda \|\nabla_\M (\psi(y)-\psi(x))\|^2].
	\end{align}}
We show later that the proposed MR-GAN architecture enjoys provable convergence guarantees.

\subsection{Manifold Regularized Training}
We provide intuitions that the objective function of the proposed MR-GAN helps in aligning the manifold of the generated data with the manifold of the real data. Let us denote the objective function of MR-GAN as 
\begin{align}
F(u,v)= \mathop{\E}_{x\sim \D_{real}}&\mathop{\E}_{y\sim \D_{\g}}[\phi (\di(x))+ \phi(1-\di(y))\nn\\
&+\lambda \|\nabla_\M (\psi(y)-\psi(x))\|^2].
\end{align}

{\textbf{Regularized Gradients.}} Note that the gradient for the generator of MR-GAN is given by
\begin{align}
\frac{\partial F(u,v)}{\partial u} = \mathop{\E}_{x\sim \D_{real}}\int_{\mathcal{Y}}\nabla_u (p_u(y) \phi(1-\di(y))\nn\\+\lambda \|\nabla_\M (\psi(y)-\psi(x))\|^2)dy,
\end{align}
where $\mathcal{Y}$ is the domain of the generated samples, and $p_u$ is the probability density function of the distribution $\D_{\g}$ for the generated samples and is dependent on $u$.  
The first term $\mathop{\E}_{x\sim \D_{real}}\int_{\mathcal{Y}}\nabla_u p_u(y) \phi(1-\di(y,\xt))dy$ follows the geometric properties of the measuring function $\phi$.
When the manifold $\M_\Y$ (where the support of $\mathcal Y$ lies) and the manifold $\M$ are far away, $\|\nabla_\M (\psi(y)-\psi(x))\|^2$ is very large. This should strongly drive $\M_\Y$ to $\M$. When $\M_\Y$ and $\M$ become closer, $\|\nabla_\M (\psi(y)-\psi(x))\|^2$ is smaller. This resembles $L^2$ optimization in general, where the loss function offers an adaptive gradient toward the optima. The gradient $\nabla_\M$ provides a {{multi-modal weighting}} and the modes of $\D_{real}$ will thus drive the gradient in training the generator.


{{\textbf{Bounded Objective Function.}}
	Additionally, recalling the objective function \eqref{objective_function}, the regularizer plays a role in the training of the generator, which has the form
	\begin{align}
	\min_{u\in \U}  \mathop{\E}_{x\sim \D_{real}}\mathop{\E}_{y\sim \D_{\g}} [\phi(1-\di(y))+\nn\\\lambda \|\nabla_\M (\psi(y)-\psi(x))\|^2].
	\end{align}
	If the embedding function $\psi$ is $L_\psi$-Lipschitz smooth on the manifold (which we assume in the sequel), we have the following inequality
	\begin{align}
	\|\nabla_\M (\psi(y)-\psi(x))\|^2  \le L_\psi^2 \|y-x\|^2,
	\end{align}
	and further we can obtain
	\begin{align}
	&	\min_{u\in \U}  \mathop{\E}_{x\sim \D_{real}}\mathop{\E}_{y\sim \D_{\g}} \!\![\phi(1\!-\!\di(y))\!+\!\lambda \|\nabla_\M (\psi(y)\!-\!\psi(x))\|^2 \nn\\
	&\le \min_{u\in \U}  \mathop{\E}_{x\sim \D_{real}}\mathop{\E}_{y\sim \D_{\g}} \phi(1-\di(y))+\lambda L_\psi^2 \|y-x\|^2].\nonumber
	\end{align}
	The regularization term $\lambda L_\psi^2 \|y-x\|^2$ imposes the similarity between the generated and real data, and thus our method penalizes dissimilarity between the generated and real data. Essentially, our method finds the generated data closer to the real one and incorporates the geometric information of the real data into the data being generated.} Later we show that $y$ does not overfit to $x$, and $y$ generalizes well with the proposed GAN architecture.

{\textbf{Training Practices.}} The objective function \eqref{objective_function} assumes that we have an infinite number of samples from $\D_{real}$ to estimate the value $\E_{x\sim \D_{real}} [\phi (\di(x,\xt))]$. In practice, the objective function $F(u,v)$ is approximated with a finite number of training samples, which is denoted by $\hat{F}(u,v)$ and is expressed as
\begin{align}
\label{mr-emp}
    \hat{F}(u,v)=\frac{1}{m}\sum_{i=1}^{m} \phi(\di(x_i)) + \phi(1-\di(y_i))\nn\\
    +\frac{1}{m^2}\sum_{i=1,j=1}^{m}\|\psi(y_i)-\psi(x_i)-\psi(y_j)+\psi(x_j)\|^2w_{ij}.
\end{align}
With finite training examples
$x_1, \ldots,x_m \sim \D_{real}$, one uses $\frac{1}{m}\sum_{i=1}^{m} \phi(\di(x_i))$ to estimate the quantity $\E_{x\sim \D_{real}} [\phi (\di(x))]$. Similarly, one can use an empirical version to estimate ${\E}_{x\sim \D_{real}}{\E}_{y\sim \D_{\g}} \phi(1-\di(y))$. The regularization term $\|\nabla_\M (\psi(y)-\psi(x))\|^2$ can be approximated as $\frac{1}{m^2}\sum_{i=1,j=1}^{m}\|\psi(y_i)-\psi(x_i)-\psi(y_j)+\psi(x_j)\|^2w_{ij}$. Here, using the definition of $w_{ij}$, we observe that if the data samples $x_i$ and $x_j$ are from different submanifolds, it encourages the output $y_i$ and $y_j$ to lie in different manifolds. Additionally, if $x_i$ and $x_j$ are from the same submanifold, it encourages the output $y_i$ and $y_i$ to lie in the same manifold. The regularizer helps in exploiting the information regarding inter- and intra-relations of the modes of the distribution of the real data, and couples $x_i$ and $y_i$  in a manifold learning fashion.

\section{Theoretical Analysis}
This section provides provability results and properties of MR-GAN\footnote{The proofs are provided in the supplementary materials.}. We first expalin assumptions that we make in our analysis, which are mild and widely used in the analysis of GANs (or DNNs) \cite{ruan2018reachability,arora2017generalization}.
\subsection{Assumptions}
\begin{as}\label{as_1}We make the following assumptions of $G_u$, $D_v$, and $\psi$.
	\begin{enumerate}[label=(\alph*)]
		\item $\forall  u,u^\prime \in \U, $ and any input $h$, $\|\g(h) - G_{u^\prime}(h)\| \le L\|u - u^\prime\|$.
		\item $\forall  u \in \U, $ and any input $h$ and $h^\prime$, $\|\g(h) -\g(h^\prime)\| \le L^\prime\|h - h^\prime\|$.
		\item The embedding function $\psi$ is $L_\psi$-Lipschitz smooth on manifold $\M$, i.e., $\|\nabla_\M \psi(y)-\nabla_\M\psi(x)\|^2 \le L_\psi^2 \|y-x\|^2$.
	\end{enumerate}	
\end{as}

Assumption~\ref{as_1}(a) means that $\g$ is $L$-Lipschitz with respect to its parameters, and we assume so for $\di$ as well. Note that this is distinct from the assumption that functions $\g, \di$ are Lipschitz (which we introduce next) which focuses on the change in function value when we change $x$, while keeping $u, v$ fixed. Assumption~\ref{as_1}(b) means that $\g$ is $L^\prime$-Lipschitz with respect to its input, and we assume so for $\di$ as well. It essentially means that a small variation in the input to the generator/discriminator does not cause a large variation in the output of the generator/discriminator.
Assumption~\ref{as_1}(c) assumes the smoothness of the embedding function in on manifold $\M$. The embedding function in practice could be an auto-encoder, which also satisfies this condition.

\begin{as}\label{as_2}We make the following assumptions about the measure function $\phi$.
	\begin{enumerate}[label=(\alph*)]
		\item $\forall x\in \M, \|\nabla_\M \phi(x)\| \le M$.
		\item The function $\phi$ is bounded in $[-\Delta, \Delta]$ in training.
		\item $\forall x, x^\prime \in \mathbb R$, $\|\phi(x) - \phi(x^\prime)\| \le L_\phi\|x - x^\prime\|$.
	\end{enumerate}	
\end{as}

Assumption~\ref{as_2}(a) is equivalent to the fact that the function $\phi$ has no geometrically step-sized property in function values.
The training of the original GAN and Wasserstein GAN uses $\phi(x)=\log(\delta+(1-\delta)x)$ and $\phi(x)=x$ respectively. Also, $\log(x)$ and $x$ are not bounded by nature when $x\in \mathbb R$. However, since $\phi$ takes input from $[0,1]$, Assumption~\ref{as_2}(b) is valid.
Assumption~\ref{as_2}(c) implies that the measure function $\phi$ is $L_\phi$-Lipschitz continuous.

\subsection{Analytical Results}
\textbf{Optimal Embedding Function.} The function $\psi$ embeds the data into a low-dimensional subspace, and thus it can prevent overfitting in the training phase. One can also use $\psi(x) = x$ to exploit the complete geometric information in the data. However, a small value of the regularizer parameter $\lambda$ is recommended in such a case to prevent overfitting, which might limit the impact of the regularizer.

One can imagine that different choices of the embedding functions can lead to different qualities of the generated data by MR-GAN. Hence, it is important to find the optimal form of the embedding function $\psi$. Since $\psi$ embeds the data into a low-dimensional subspace, an interesting case to study is to find the optimal one-dimensional form of $\psi$. As the regularized objective only takes effect in the training of the generator, we can write the joint optimization of finding the best generator and the embedding function $\psi$ in an empirical fashion as
\begin{align}\label{emp_cost_gen_embd}
\min_{u\in \U, \psi}  \frac{1}{m}\sum_{i=1}^m [\phi(1-\di(y_i))
+\lambda \|\nabla_\M (\psi(y_i)-\psi(x_i))\|^2].
\end{align}
We provide the result in the following theorem for such a case.
\begin{thm}\label{embedding}
	The optimal one-dimensional embedding function $\psi(x)$ exists and admits the following representation
	\begin{align}
	\psi(x)=\sum_{i=1}^{m}\alpha_i K(x_i,x),
	\end{align}
	where $K: \mathbb{R}^d \times \mathbb{R}^d \rightarrow \mathbb{R}$ is a Mercer kernel.
\end{thm}
Here, we still need to find the coefficients $\alpha_i$ for the finite dimensional space. A good method is indicated in \cite{belkin2006manifold}. First, one fixes the type of kernel function $K$ and optimizes the problem \eqref{emp_cost_gen_embd} with respect to $\alpha_i$. Then, the optimal $\alpha_i$ can be easily found using simple first-order derivative methods. 

For the embedding function $\psi$ which gives the embedding with a dimension that is higher than $1$, one can use auto-encoders to encode the high-dimensional data into a low-dimensional subspace.

\textbf{Generalization.}
Since we can only access (and optimize) the empirical distance between the distributions in practice, it becomes important to ensure that this empirical distance is close to the true distance for the generated and real distributions. As the training algorithm is supposed to run in polynomial time, one has to estimate the true distance using only polynomial number of samples \cite{arora2017generalization}. Indeed, it is shown in \cite{arora2017generalization} that if we do not have enough samples for training the GAN, $1)$ the distance between the empirical distributions can be close to the maximum possible distance even if the samples are drawn from the same distribution; $2)$ even if the generator happens to find the real distribution, the
distance between the empirical distributions can still be large and the generator has no idea that
it has succeeded.

Thus, it is crucial to answer the following question: can MR-GAN generalize well and approximate the true distance between the generated and real distributions with a reasonable number of samples? Our following result answers this in affirmative. 
\begin{thm}\label{thm_realizability}
	Let $\hat{\D}_{real}$ and $\hat{\D}_\g$ be empirical versions with at least $m$ samples each for MR-GAN. There is a universal constant $C$ such that when $m \ge \frac{Cp\log(LL_\phi p/\epsilon)(\Delta+4\lambda M^2)^2}{\epsilon^2} $, we have with probability at least $1-\exp (1-p)$: 
	\begin{align}
	|F(u,v)-\hat F(u,v)|\le \epsilon.
	\end{align}
\end{thm}
{{The above theorem shows that if a sufficient amount of training data is available, the distance between the empirical objective function $\hat F(u,v)$ and the population objective function $F(u,v)$ is sufficiently small.} This result is important as it guarantees that the analysis of MR-GAN conducted based on the population objective function can be well generalized to the empirical form. Thus, it ensures that the theoretical guarantees of MR-GAN can be well satisfied in practice.}

\textbf{Existence of Equilibrium.} The training of GAN has the goal to end up with an equilibrium for the min-max game between the generator and the discriminator. That is, the discriminator outputs $1/2$ for both the cases where the input is the real data or the generated data, which essentially means that the discriminator guesses randomly and cannot distinguish between real and generated data. On the other hand, the generator cannot exploit the output of the discriminator by back-propagation and cannot update itself and improve the quality of the generated data anymore. Therefore, ensuring the existence of the equilibrium of a certain GAN architecture is crucial before training process starts. It is important that we have provable results showing the existence of equilibrium for the proposed MR-GAN. Interestingly, if we change the regularizer in the objective function \eqref{problem} from $\|\nabla_\M (\psi(y_i)-\psi(x_i))\|^2$ to $\|\nabla_\M \psi(y_i)\|^2$, which is used in conventional manifold learning problems, our analysis indicates that the equilibrium cannot be guaranteed. To show the existence of equilibrium for the proposed MR-GAN architecture, we use the following definition of the $\epsilon$-approximate equilibrium.

\begin{definition}[\cite{arora2017generalization}]
	A pair of mixed strategies $(\mathcal S_u,\mathcal S_v)$ is an $\epsilon$-approximate equilibrium, if for some value $V$
	\begin{align}
	&\forall v\in\V, \mathop\E_{u\sim \mathcal{S}_u}F(u,v) \le V+\epsilon;\\
	& \forall u\in\U, \mathop\E_{v\sim \mathcal{S}_v}F(u,v) \ge V-\epsilon.
	\end{align}
	If the strategies $\mathcal S_u,\mathcal{S}_v$ are pure strategies, then this pair is called an $\epsilon$-approximate pure equilibrium.
\end{definition}
\begin{thm}\label{thm_equilibrium}
	If the generator can approximate any point mass by $\Eb_{h\sim \D_h}[\|\g(h)-x\|] \le \epsilon$, there is a universal constant $C > 0$ such that for any $\epsilon$, there exists $T=\frac{C\Delta^2 p \log (LL^\prime L_\phi p/\epsilon)}{\epsilon^2}$ generators $G_{u1}, \ldots, G_{uT}$. Let $\mathcal{S}_u$ be a uniform distribution on $u_i$, and $D$ is a discriminator that outputs $1/2$, then $(\mathcal S_u, D)$ is an $\epsilon$-approximate equilibrium for MR-GAN.
\end{thm}

Note that in the above result, the generator uses mixed strategies, which means that the generated data comes from a mixture of generators. One can add an output layer of ReLU activation functions to the generators to construct an integrated neural network of the generator, and the output is uniformly distributed over the results from the $T$ generators in the theorem. One possible construction can be found in the Lemma $4$ in \cite{arora2017generalization}. 

\textbf{Stable Training.}
From both the theoretical and the practical perspectives, the training of GAN remains a challenging problem, one of which is the issue of instability in optimizing GANs. It is presented in \cite{nagarajan2017gradient} that the training dynamics in ``(stochastic) gradient descent'' form of GANs optimization can be well analyzed by the method of nonlinear differential equations (ODE), thus providing a characterization of ``stability'' of GAN training. It is important to show that MR-GAN also falls into this general framework to characterize the training dynamics, and to show that the proposed MR-GAN can stabilize the training process.

Assuming that the generator and discriminator networks are parameterized by the sets of parameters,
$u$ and $v$, respectively, we investigate the problem of analyzing stability of approaches based on stochastic gradient descent to solve~\eqref{problem}. That is, we take simultaneous gradient steps in both $u$ and $v$.

All our conditions
are imposed on both $(u^\ast, v^\ast)$ and all equilibrium points in a small neighborhood around it.
Given the above consideration, our focus is on proving the stability of the dynamical system around
equilibrium points, i.e., points $\theta^\ast$ for which $h(\theta^\ast)=0, h(\theta)=\nabla F(\theta)$. 
We now discuss conditions under which we can guarantee exponential stability, which is originally defined for a dynamic system as follows.
\begin{definition}[\textbf{Stability}]
	(\cite{khalil1996noninear} Consider a system consisting of variables $\theta\in \mathbb{R}^n$ whose time derivative is defined by $h(\theta)$ as
	\begin{align}\label{ode}
	h(\theta)=\nabla F(\theta).
	\end{align}
	Let $\theta(t)$ denote the state of the system at some time $t$. Then an equilibrium point of the system in Equation \eqref{ode} is 
	\begin{itemize}
		\item stable if for each $\epsilon>0$, there is $\delta=\delta(\epsilon)>0$ such that
		\begin{align}
		\|\theta(0)\|\le \delta, \|\theta(t)\|\le \epsilon, \forall t\ge 0.
		\end{align}
		
		\item asymptotically stable if it is stable and $\delta >0$ can be chosen such that
		\begin{align}
		\|\theta(0)\|\le \delta, \lim_{t\rightarrow \infty } \theta(t) =0.
		\end{align}
		
		\item
		exponentially stable if it is asymptotically stable and $\delta, k, \lambda >0$ can be chosen such that 
		\begin{align}
		\|\theta(0)\|\le \delta, \|\theta(t)\|\le k \|\theta(0)\|\exp(-\lambda t).
		\end{align}
	\end{itemize}
\end{definition}

\begin{figure*}
    \centering
     \includegraphics[width=0.95\textwidth]{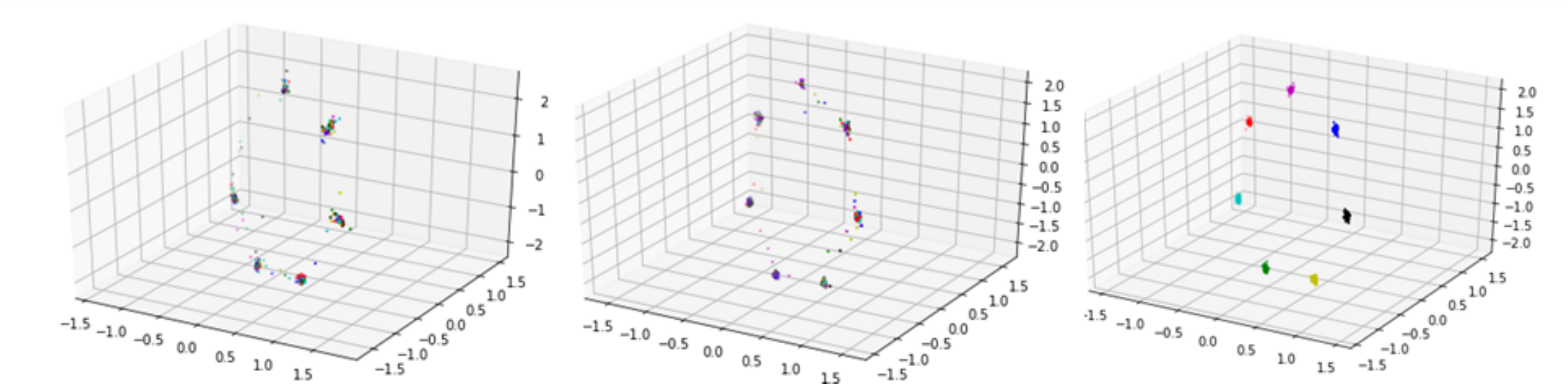}
    \caption{MR-GAN avoids the mode collapse problem and generalizes better on a toy 2D mixture of Gaussians
dataset in a 3D ambient space. (a) original GAN, (b) MR-GAN, (c) Ground truth.} 
    \label{fig:Gaussian_Mixture}
\end{figure*}

Specifically, we invoke the well known linearization theorem \cite{khalil1996noninear} analyzed for GANs training dynamics \cite{nagarajan2017gradient}, which states that if the Jacobian of the dynamical
system $\mathbf{J} = \partial h(\theta)/\partial \theta|\theta=\theta^\ast$ evaluated at an equilibrium point is Hurwitz (which has all strictly negative
eigenvalues, $Re(\lambda_i(\mathbf J)) < 0, \forall i = 1, \ldots , n),$ then the optimization of the GAN system training will converge to $\theta^\ast$	for some non-empty region around $\theta^\ast$, at an exponential rate. This means that the system is locally asymptotically stable, or more precisely, locally exponentially stable.
Thus, an important contribution here is a proof of the following fact: under some
conditions, the Jacobian of the dynamical system given by the proposed GAN update is a Hurwitz matrix at equilibrium. For simplicity, we denote the equilibrium point in the min-max game of the GAN training by $(u^\ast, v^\ast)$, which are the parameter sets of the discriminator and the generator at the equilibrium points. Recall that
\begin{align}
F(u,v)= \mathop{\E}_{x\sim \D_{real}}&\mathop{\E}_{y\sim \D_{\g}}[\phi \di(x)+ \phi(1-\di(y))\nn\\&+\lambda \|\nabla_\M (\psi(y)-\psi(x))\|^2].
\end{align}
The gradient steps in both $u$ and $v$ are taken simultaneously, resulting in the following gradient differential equations
\begin{align}\label{updates}
\dot{u} = \nabla_u F(u,v), \ \ \ \dot{v} = \nabla_v F(u,v).
\end{align}

\begin{thm}\label{thm_stable}
	The dynamical system defined by the MR-GAN objective in Equation \eqref{problem} and the updates in
	Equation \eqref{updates} is locally exponentially stable with respect to an equilibrium point $(u^\ast, v^\ast)$.
\end{thm}
This shows that the proposed MR-GAN is locally exponentially stable. That is, for some region around an equilibrium of the updates, the gradient updates will converge
to this equilibrium at an exponential rate. As an interesting note, Wasserstein GANs~\cite{arjovsky2017wasserstein} are not even asymptotically stable \cite{nagarajan2017gradient}. However, adding the manifold regularization term makes it locally exponentially stable.
Different from the analysis in \cite{nagarajan2017gradient} and the proposed GAN training architecture therein, we can guarantee the existence of the equilibrium, around which the training has stable convergence. However, the analysis in \cite{nagarajan2017gradient} does not guarantee so, as it only demonstrates convergence if there do exist points that satisfy certain criteria. Thus, we provide a systematic analysis that proves the existence of the equilibrium and the stable convergence.

\section{Experiments}
In this section, we corroborate our theories using synthetic data
and real data sets of increasing complexity. We present the performance comparison to widely used benchmark GAN architectures\footnote{Additional results are provided in supplementary materials.}. 

We employ the recently proposed {\emph{Geometry Score} (GS)} metric \cite{khrulkov2018geometry} for assessing the quality of generated samples and detecting various levels of failure models. GS compares the topological properties of the underlying real data manifold and the generated
one, which provides both qualitative and quantitative means for evaluation the results generated by GANs. It was shown in \cite{khrulkov2018geometry} that GS is more expressive in capturing various failure modes of GANs compared to its conventional counterparts, such as, \emph{Inception Score}~\cite{salimans2016improved}, and \emph{Fr\'echet Inception Distance}~\cite{heusel2017gans}. A lower value of GS indicates a better match between the generated data and the real data.

\subsection{Synthetic Data}
To illustrate the impact of the proposed regularization in the training of the generator, we train the original GAN architecture \cite{goodfellow2014generative} (using Adam Optimizer with a learning rate of $\gamma=1e-3$ for both networks) on a 2D mixture of $8$ Gaussians evenly arranged in a circle. However, the circle of the Gaussian mixture lies in a hyperplane in a 3D space. We show this dataset model in the third subfigure in Figure~\ref{fig:Gaussian_Mixture}. Therefore, the generator has to search for 2D submanifolds in a 3D space. The first two subfigures in Figure~\ref{fig:Gaussian_Mixture} show the GAN training result of this model after $10,000$ training iterations. We present the result of the original GAN in the first subfigure and that of MR-GAN in the second. For the MR-GAN, we set the kernel  scale parameter $\rho=128$ and the regularization parameter $\lambda=0.5$. We can clearly observe from the comparison in the figure that the original GAN misses one of the $8$ modes and the problem of mode collapse happens, and the proposed MR-GAN learns to evenly spread the probability mass and converges to all the $8$ modes without any mode collapse. Second, we can see from the figure that the data mass generated by the proposed GAN architecture lies heavily within the mode, and the probability mass resembles the real probability in the third subfigure very well. However, the data mass generated by the original GAN scatters around the mode, compared to the result generated by MR-GAN. Furthermore, the results generated by the original GAN have a GS of $0.909$ and the results generated by the proposed MR-GAN architecture have a GS of $0.442$, which is an improvement of $51.4\%$.

\subsection{MNIST Dataset}
In our second experiment, we test our approach on the {{MNIST dataset~\cite{mnist_data}}} of handwritten digits. We compare the proposed GAN architecture based on the recently proposed model, i.e., Wasserstein GAN (WGAN)~\cite{arjovsky2017wasserstein}. We use the RMSProp Optimizer with a learning rate of $\gamma=1e-4$ for both networks. In the following tables, we quantify the performance in terms of the GS for the proposed MR-WGAN architecture with different values of kernel scale parameter $\rho$ and with different values of regularization parameter $\lambda$, after 300K training iterations. 

First, we set the kernel scale parameter $\rho$ to $6.4$ and vary the value of the regularization parameter $\lambda$ from $0.05$ to $0.4$, as shown in Table~\ref{rho64}. Note that the results of WGAN yields GS of $0.414$, which is shown with $\lambda=0$. When the value of $\lambda$ is small, we observe the improvement in GS for MR-WGAN compared to the results of WGAN. When $\lambda=0.2$, the proposed MR-WGAN has results with GS $=0.384$, which provides an improvement of $7.25\%$ in GS. We also provide various GS results when $\rho=10$ in Table~\ref{rho10}. When $\lambda$ is small, we again observe improvement. When $\lambda=0.01$, we have the best result and the GS of the results generated by MR-GAN is $0.372$, which is an improvement of $10.14\%$.
\begin{table}[h]
\centering
\caption{Geometry Score for the MNIST MR-WGAN ($\rho=6.4$)}
\begin{tabular}{llllllll}
\hline
$\lambda$ & 0.05 & 0.1 & 0.2 & 0.3 & 0.4  &0\\
GS & 0.405 & 0.403 & {\bf0.384} & 0.444 & 0.441 & {\textit {0.414}}\\\hline
\label{rho64}
\end{tabular}
\end{table}
\vspace{-0.2in}
\begin{table}[h]
\centering
\caption{Geometry Score for the MNIST MR-WGAN ($\rho=10$)}
\begin{tabular}{llllllll}
\hline
$\lambda$ & 0.005 & 0.01 & 0.02 & 0.1 & 0.2 & 0 \\
GS & 0.382 & {\bf 0.372} & {0.379} & 0.404 & 0.506 & {\textit{0.414}}\\\hline
\label{rho10}
\end{tabular}
\end{table}

\vspace{-0.1in}
\subsection{CelebA Dataset}
In this experiment, we use the CelebA dataset \cite{liu2015faceattributes}, which is comprised of $202,599$ images of celebrity faces. We trained a MR-DCGAN~\cite{radford2015unsupervised} using $90\%$ of the images from CelebA dataset. As in the DCGAN case, we rescale the data to lie in the range $[-1,1]$. We use the same architecture as the DCGAN implementation in the discriminator and generator networks. We also use the Adam Optimizer with a learning rate of $\gamma = 2\mathrm{e}{-4}$ for both networks. For the embedding function $\psi$, we use a convolutional autoencoder that embeds the training set of the CelebA dataset into a $100$ dimensional latent space.  

 We train the network with different values of $\lambda, \rho$ as explained earlier and report the quality of each GAN in Table~\ref{tab:dcgan}. We see that adding the proposed manifold regularization significantly improves the performance of the DCGAN (shown with $\lambda,\rho = (0.0, 0.0)$), leading to a geometry score that is lower by about $\sim 40\%$. Samples from the MR-DCGAN are shown in Figure \ref{fig:celeba_results}.

\begin{table}[h]
\centering
\caption{Geometry Score $(\times 1\mathrm{e}{-3})$ for the CelebA MR-DCGAN}
\begin{tabular}{p{1.3cm}p{1.3cm}p{1.3cm}p{1.3cm}p{1.3cm}}
\hline
$\lambda,\rho$ & (0.1, 0.3) & (0.5, 0.3) & (0.1, 0.5)  & (0.0,0.0) \\
GS &11.95 &25.20 &  \textbf{6.33} & {\textit{10.77}} \\\hline
\label{tab:dcgan}
\end{tabular}
\end{table}
\begin{figure}
    \centering
    \includegraphics[trim={8.0cm 0cm 7cm 0cm},clip,width=0.9\linewidth]{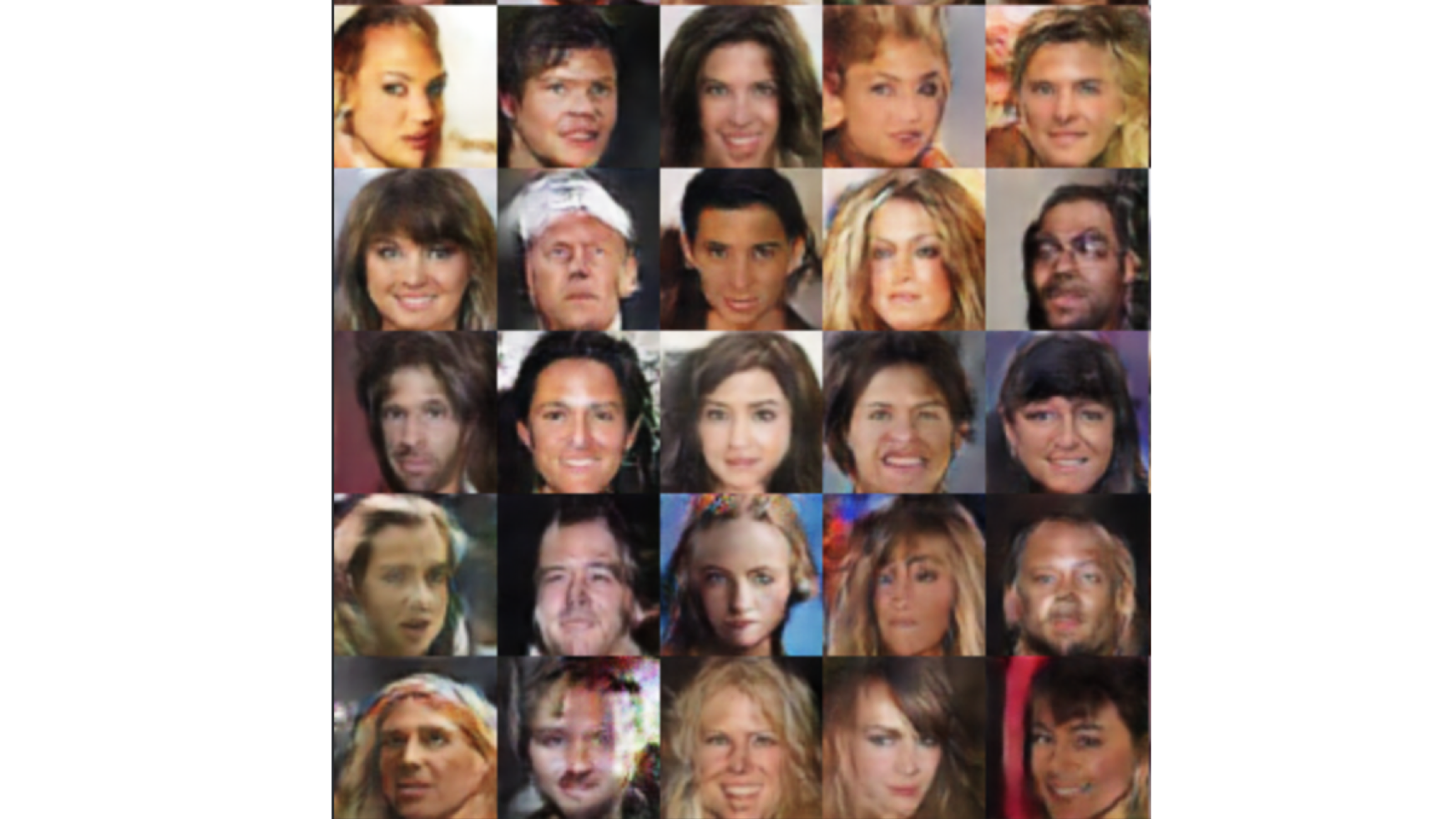}
    \caption{\small{Randomly generated samples from the MR-DCGAN.}}
    \label{fig:celeba_results}
\end{figure}

\vspace{-0.1in}
\section{Conclusion}
We studied the problem of training GANs. We proposed a manifold regularization method to force the generator to respect the unique manifold geometry of the real data in order to generate high quality data. Furthermore, we theoretically proved that the incorporation of this regularization term in any class of GANs leads to improved performance in terms of generalization, existence of equilibrium, and stability. We empirically showed that the proposed manifold regularization helps in avoiding mode collapse and leads to stable training. There are still many interesting questions that remain to be explored in the future such as establishing the global convergence properties of GAN training. It will also be interesting to explore the connection between the proposed method and the recently proposed Jacobian Clamping method~\cite{jacob_clip}. Other questions such as the case where both the discriminator and the generator are regularized can also be interesting to investigate.

{\small
	\bibliographystyle{ieee}
	\bibliography{GAN}
}

\appendix
\onecolumn
\section*{Supplementary Material for MR-GAN }
\begin{figure}[h]
    \centering
    \includegraphics[width=.5\textwidth]{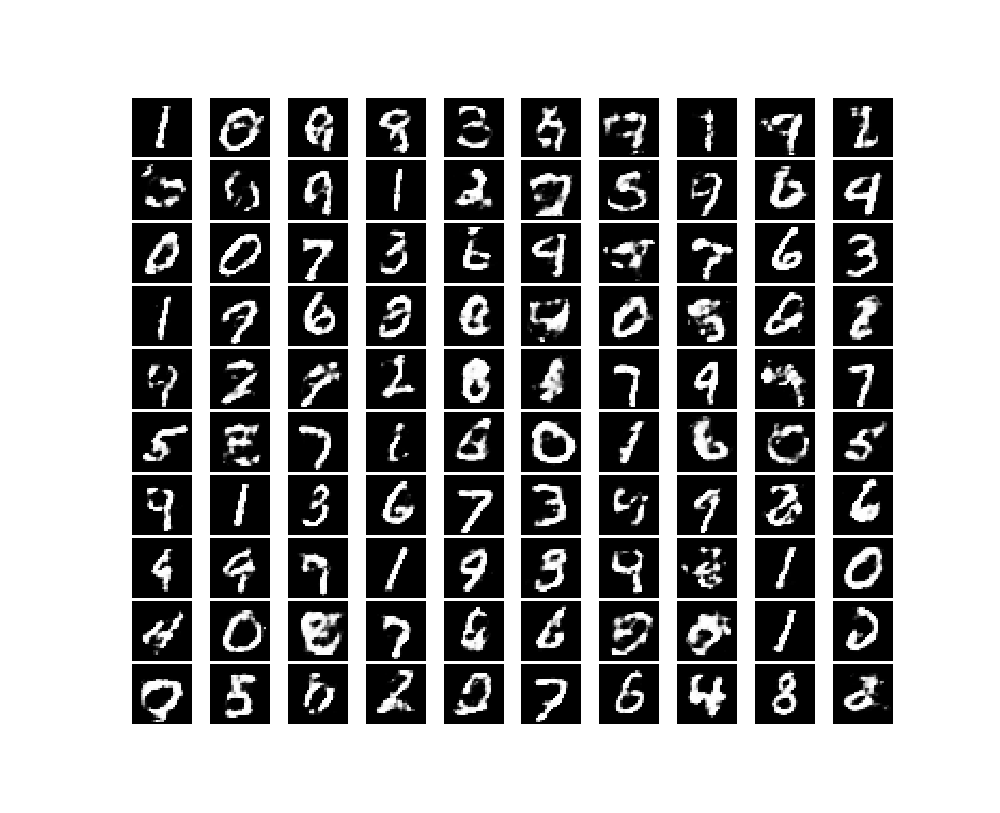}
    \caption{\small{Randomly generated digits from the MR-WGAN.}}
    \label{fig:mnist_sm}
\end{figure}
In Figure~\ref{fig:mnist_sm}, we present the results of the proposed MR-WGAN for the MNIST dataset. The resutls are obtained after 10,000 training iterations.
\section{Proof of Theorem \ref{embedding}}
\begin{thm*}[\bf 1]
	The optimal one-dimensional embedding function $\psi(x)$ exists and admits the following representation
	\begin{align}
	\psi(x)=\sum_{i=1}^{m}\alpha_i K(x_i,x),
	\end{align}
	where $K: \mathbb{R}^d \times \mathbb{R}^d \rightarrow \mathbb{R}$ is a Mercer kernel.
\end{thm*}
\begin{proof}
	Consider samples $y_i$ and $x_i$ (same index $i$), $i=1,\ldots,m$, are drawn from the generator and the real data set, we use the empirical expressions. To find the optimal embedding function $\psi$, we write the cost function for the generator
	\begin{align}\label{emp_cost_gen_embd1}
	\min_{u\in \U, \psi}  \frac{1}{m}\sum_{x\sim \D_{real},y\sim \D_{\g}} [\phi(1-\di(y_i))
	+\lambda \|\nabla_\M (\psi(y_i)-\psi(x_i))\|^2].
	\end{align}
	By using the samples, any function $\psi$ derived can be uniquely decomposed into a component $\psi_{\parallel}$ in the linear space spanned by the kernel functions $\{K(x_i,\cdot)\}_{i=1}^{m}$, and a component $\psi_\perp$ orthogonal to it. Thus,
	\begin{align}
	\psi = \psi_{\parallel}+ \psi_\perp = \sum_{i=1}^{m}\alpha_i K(x_i,\cdot) + \psi_\perp.
	\end{align}
	By the reproducing property, the evaluation of $\psi$ on any data point $x_j$ is independent of the orthogonal component $\psi_\perp$:
	\begin{align}
	\psi(x_j) = \langle f, K(x_j,\cdot)\rangle = \langle \sum_{i=1}^{m}\alpha_i K(x_i,\cdot), K(x_j,\cdot)\rangle + \langle \psi_\perp, K(x_j,\cdot)\rangle.
	\end{align}
	Since the second term zeros out, and $\langle (x_i,\cdot), K(x_j,\cdot)\rangle = K(x_i,x_j)$, it follows that $\psi(x_j) = \sum_{i=1}^{m}\alpha_i K(x_i,x_j)$.
	
	Indeed, we find that
	\begin{align}
	FLF^T = \langle \psi(y), L\psi(y) \rangle - 2\langle \psi(y), L\psi(x)\rangle+ \langle \psi(x), L\psi(x) \rangle
	\end{align}
	where $F=[f_1,f_2,\ldots,f_m]$ and $f_i=\psi(y_i)-\psi(x_i)$, $\psi(x)=[\psi(x_1),\psi(x_2),\ldots,\psi(x_m)]$, and $\psi(y)=[\psi(y_1),\psi(y_2),\ldots,\psi(y_m)]$. Hence, it can be further written with respect to $L$-norm as
	\begin{align}
	FLF^T = \sum_{j=1}^{m}\|\sum_{i=1}^{m}\alpha_i K(x_i,y_j)\|^2_L + \sum_{j=1}^{m}\|\psi_\perp(y_j)\|^2_L +\sum_{j=1}^{m}\|\sum_{i=1}^{m}\alpha_i K(x_i,x_j)\|^2_L - \langle(\psi_\parallel(x),L\psi_\perp(y) \rangle,
	\end{align}
	where $\psi_\parallel(x) =[\psi_\parallel(x_1),\psi_\parallel(x_2),\ldots,\psi_\parallel(x_m)]$, and $\psi_\parallel(y) =[\psi_\parallel(y_1),\psi_\parallel(y_2),\ldots,\psi_\parallel(y_m)]$.
	It follows that the optimal embedding function $\psi$ of problem \ref{emp_cost_gen_embd1} must have $\psi_\perp=0$. Therefore, it admits a representation 
	\begin{align}
	\psi(x)=\sum_{i=1}^{m}\alpha_i K(x_i,x).
	\end{align}
\end{proof}

\section{Proof of Theorem \ref{thm_realizability}}
\begin{thm*}[\bf 2]
	Let $\hat{\D}_{real}$ and $\hat{\D}_\g$ be empirical versions with at least $m$ samples each for MR-GAN. There is a universal constant $C$ such that when $m \ge \frac{Cp\log(LL_\phi p/\epsilon)(\Delta+4\lambda M^2)^2}{\epsilon^2} $, we have with probability at least $1-\exp (1-p)$: 
	\begin{align}
	|F(u,v)-\hat F(u,v)|\le \epsilon.
	\end{align}
\end{thm*}
\begin{proof}
	Let $\mathcal X$ be a finite set such that every point in $\V$ is within distance $\epsilon/8LL_\phi$ of a point in $X$ (a so-called $\epsilon/8LL_\phi$-net). Standard constructions give an $X$ satisfying $\log |\mathcal X| \le O(p\log(LL_\phi p/\epsilon))$. For every $v \in \mathcal{X}$, by Hoeffding inequality we know
	\begin{align}
	\Pr\left(|f(\D_{real}, \D_\g,v)-f(\hat{\D}_{real},\hat\D_\g,v)| \ge \frac{\epsilon}{4}\right) &\le 2\exp\left( -\frac{m^2\frac{\epsilon^2}{16}}{m(2\Delta+8\lambda M^2)^2}\right)\\ &=2\exp\left(-\frac{m\epsilon^2}{32(\Delta+4\lambda M^2)^2} \right) 
	\end{align}
	where $f(\D_{real},\D_\g,v)=\mathop\E_{x\sim \D_{real}}\mathop\E_{y\sim \D_\g}\phi(1-\di(y))+\lambda\|\nabla_\M(\psi(y)-\psi(x)\|^2$.
	Thus, we can union bound over all $v \in \mathcal{X}$ for large enough constant $C$
	\begin{align}
	\Pr\left(|f(\D_{real},\D_\g,v)-f(\hat{\D}_{real},\hat\D_\g,v)| \ge \frac{\epsilon}{4}\right) \le& 2|\mathcal{X}|\exp\left(-\frac{m\epsilon^2}{32(\Delta+4\lambda M^2)^2} \right)\\
	=&\exp\left(\log 2|\mathcal{X}|-\frac{m\epsilon^2}{32(\Delta+4\lambda M^2)^2} \right) \\
	\le & \exp\left(Cp\log(LL_\phi p/\epsilon) -\frac{m\epsilon^2}{32(\Delta+4\lambda M^2)^2}\right) .
	\end{align}
	Choose $m$ such that $m\ge \frac{Cp\log(LL_\phi p/\epsilon)(\Delta+4\lambda M^2)^2}{\epsilon^2}$, and thus, with high probability (at leat $1-\exp(-p)$) we have $|f(\D_{real},\D_\g,v)-f(\hat{\D}_{real},\hat\D_\g,v)| \le \frac{\epsilon}{4}$.
	
	Now, for $v\in \V$ and $v^\prime \in \mathcal X$ such that $\|v-v^\prime\|\le \epsilon/8LL_\phi$,
	we have
	\begin{align}
	|f(\D_{real},\D_\g,v)-f(\hat{\D}_{real},\hat\D_\g,v)| \le& |f(\D_{real},\D_\g,v^\prime)-f(\hat{\D}_{real},\hat\D_\g,v^\prime)|\\
	&+|f(\D_{real},\D_\g,v^\prime)-f(\D_{real},\D_\g,v)|\\
	&+|f(\hat \D_{real},\hat\D_\g,v^\prime)-f(\hat\D_{real},\hat\D_\g,v)|\\
	\le &\epsilon/4+\epsilon/8+\epsilon/8=\epsilon/2.
	\end{align}
	The value of $\epsilon/8$ results from Lipschitz continuity.
	
	Similarly, we can bound
	\begin{align}
	|	\mathop{\E}_{x\sim \D_{real}}\phi \left( \di(x,\xt)\right) -\mathop{\E}_{x\sim \hat\D_{real}}\phi \left( \di(x,\xt)\right)| \le \epsilon/2
	\end{align}
	with $m\ge \frac{Cp\log(LL_\phi p/\epsilon)\Delta^2}{\epsilon^2}$. Choose $m$ such that $m \ge \frac{Cp\log(LL_\phi p/\epsilon)(\Delta+4\lambda M^2)^2}{\epsilon^2}$.
\end{proof}

\section{Proof of Theorem \ref{thm_equilibrium}}
\begin{thm*}[\bf 3]
	If the generator can approximate any point mass by $\Eb_{h\sim \D_h}[\|\g(h)-x\|] \le \epsilon$, there is a universal constant $C > 0$ such that for any $\epsilon$, there exists $T=\frac{C\Delta^2 p \log (LL^\prime L_\phi p/\epsilon)}{\epsilon^2}$ generators $G_{u1}, \ldots, G_{uT}$. Let $\mathcal{S}_u$ be a uniform distribution on $u_i$, and $D$ is a discriminator that outputs $1/2$, then $(\mathcal S_u, D)$ is an $\epsilon$-approximate equilibrium for MR-GAN.
\end{thm*}
\begin{proof}
	We first prove the value of the function $F(u,v)$ of the game at the equilibrium must be equal to $2\phi(1/2)$. This strategy has payoff $2\phi(1/2)$ no matter what the generator does, so $V \ge 2\phi(1/2)$.
	
	For the generator, we use the assumption that for any point $x$ and any $\epsilon >0$, there is a generator (which we dentoe by $G_{x,\epsilon}$) such that $\Eb_{h\sim \D_h}\|G_{x,\epsilon}(h)-x\|\le \epsilon$. Now for any $\alpha >0$, consider the following mixture of generators: sample $x \sim \D_{real}$, then use the generator $G_{x,\alpha}$. Let $\D_\alpha$ be the distribution generated by this mixture of generators. The Wasserstein distance between $\D_\alpha$ and $\D_{real}$ is bounded by $\alpha$. Since the discriminator is $L^\prime$-Lipschitz, it cannot distinguish between $\D_\alpha$ and $\D_{real}$. In particular we know for any discriminator $\di$
	\begin{align}
	|\mathop\Eb_{y\sim \D_\alpha}[\phi(1-\di(y))]-\mathop\Eb_{x\sim \D_{real}}[\phi(1-\di(x))]| \le O(L_\phi L^\prime \alpha).
	\end{align}
	Therefore,
	\begin{align}\label{itm1}
	&\max_{v\in \V} \mathop\Eb_{y\sim \D_\alpha}\mathop\Eb_{x\sim \D_{real}}[\phi(\di(x))]+[\phi(1-\di(y))]+\lambda \|\nabla_\M (\psi(y)-\psi(x))\|^2 \nonumber\\
	& \le \max_{v\in \V} \mathop\Eb_{y\sim \D_\alpha}\mathop\Eb_{x\sim \D_{real}}-\phi(1-\di(x))+\phi(1-\di(y))+\lambda \|\nabla_\M (\psi(y)-\psi(x))\|^2 \nonumber\\&+ \max_{v\in \V} \mathop\Eb_{x\sim \D_{real}}\phi(\di(x))+\phi(1-\di(x)) \nonumber\\
	& \le O(L_\phi L^\prime \alpha) +\lambda L_{\psi}^2\alpha^2+2\phi(1/2)
	\end{align}
	Here the last step uses the assumption that $\phi$ is concave. Therefore the value is upperbounded by $V\le O(L_\phi L^\prime \alpha) +\lambda L_{\psi}^2\alpha^2+2\phi(1/2)$ for any $\alpha$. Taking limit of $\alpha$ to $0$, we have $V=2\phi(1/2)$.
	
	The value of the game is $2\phi(1/2)$ in particular means the optimal discriminator cannot do anything other than a random guess. Therefore we will use a discriminator that outputs $1/2$. Next we will construct the generator.
	
	Let $\{\mathcal{S}^\prime_u,\mathcal{S}^\prime_v\}$ be the pair of optimal mixed strategies as in the theorem and $V$ be the optimal value. We will show that randomly sampling $T$ generators from $\mathcal{S}^\prime_u$ gives the desired mixture with high probability. 
	
	Construct $\epsilon/4LL^\prime L_\phi$-nets $\V$ for the parameters of the discriminator (for any $v, v^\prime \in \V, \|v-v^\prime\| \le \epsilon/4LL^\prime L_\phi$). By standard construction, the sizes of these $\epsilon$-nets satisfy $\log|\V| \le C^\prime p\log(LL^\prime L_\phi p/\epsilon) $ for some constant $C^\prime$. Let $u_1,\ldots,u_T$ be independent samples from $\mathcal{S}^\prime_u$. By Hoeffding's inequality, for any $v\in \V$, we know
	\begin{align}
	P(\mathop\Eb_{i\in[T]}F(u_i,v) -\mathop\Eb_{u\in \U }F(u,v) \ge \frac{\epsilon}{2} ) \le \exp \left(-\frac{2T^2\frac{\epsilon^2}{4}}{T4\Delta^2}\right) =\exp \left(-\frac{T\epsilon^2}{8\Delta^2}\right)
	\end{align}
	
	Now for all $v\in \V$, with union bound, we have
	\begin{align}
	\forall v \in \V,	&P(\mathop\Eb_{i\in[T]}F(u_i,v) -\mathop\Eb_{u\in \U }F(u,v) \ge \frac{\epsilon}{2} ) \le  |\V|\exp \left(-\frac{T\epsilon^2}{8\Delta^2}\right)\\
	& \le \exp \left(C^\prime p\log(LL^\prime L_\phi p/\epsilon)-\frac{T\epsilon^2}{8\Delta^2}\right).
	\end{align}
	Thus, when $T=\frac{C\Delta^2p\log(LL^\prime L_\phi p/\epsilon)}{\epsilon^2}$ and $C\ge 8C^\prime$, with high probability
	\begin{align}\label{itm2}
	\mathop\Eb_{i\in[T]}F(u_i,v) \le \mathop\Eb_{u\in \U }F(u,v) + \frac{\epsilon}{2} .
	\end{align}
	
	By construction of the net, we have $\|v-v^\prime\|\le \frac{\epsilon}{4LL^\prime L_\phi}$. 
	It is easy to find that $F(u,v)$ is $2LL^\prime L_\phi$-Lipschitz with respect to $v$, and therefore,
	\begin{align}
	\mathop\Eb_{i\in[T]}F(u_i,v^\prime) \le 	\mathop\Eb_{i\in[T]}F(u_i,v)+2LL^\prime L_\phi \frac{\epsilon}{4LL^\prime L_\phi} =\mathop\Eb_{i\in[T]}F(u_i,v)+ \frac{\epsilon}{2}
	\end{align}
	Together with the inequality \eqref{itm2}, we obtain 		\begin{align}
	\forall v^\prime \in \V,	\mathop\Eb_{i\in[T]}F(u_i,v^\prime) \le 2\phi(1/2)+\epsilon.
	\end{align}
	
	This means the mixture of generators can win against any discriminator. By probabilistic argument we know there must exit such generators. The discriminator (outputs 1/2) obviously achieve value $V$ no matter what the generator is. Therefore we get an approximate equilibrium.
\end{proof}

\section{Proof of Theorem \ref{thm_stable}}
\begin{thm*}[\bf 4]
	The dynamical system defined by the MR-GAN objective in Equation \eqref{problem} and the updates in
	Equation \eqref{updates} is locally exponentially stable with respect to an equilibrium point $(u^\ast, v^\ast)$.
\end{thm*}
\begin{proof}
	To derive the Jacobian, we begin with subtly different algebraic form of the GAN objective by 
	\begin{align}
	F(u,v)= \mathop{\E}_{x\sim \D_{real}}\int_{\mathcal{Y}}(p_u(y)[\phi \di(x)+ \phi(1-\di(y))]+\lambda \|\nabla_\M (\psi(y)-\psi(x))\|^2)dy.
	\end{align}
	Thus, we have the following form of the dynamic ODE system
	\begin{align}
	&\dot{u}=-\frac{\partial F(u,v)}{\partial u} = -\mathop{\E}_{x\sim \D_{real}}\int_{\mathcal{Y}}(\nabla_u p_u(y) \phi(1-\di(y))+\lambda \|\nabla_\M (\psi(y)-\psi(x))\|^2)dy\\
	&\dot{v} =\frac{\partial F(u,v)}{\partial v} = \mathop{\E}_{x\sim \D_{real}}\mathop{\E}_{y\sim p_u(y)}[\phi^\prime \di(x)\nabla_v
	\di(x)- \phi^\prime(1-\di(y))\nabla_v
	\di(y)].
	\end{align}
	
	The Jocabian matrix $\J$ consists of blocks as
	\begin{align}
	\J = \begin{pmatrix}
	\J_{vv} & \J_{vu}\\\
	\J_{uv} & \J_{uu}
	\end{pmatrix}
	\end{align}
	Then $\J_{vv}$ is:
	\begin{align}
	\J_{vv} =& \left.\nabla_v^2 F(u,v)\right|_{u=u^\ast,v=v^\ast}=\left.\frac{\partial \dot{v}}{\partial v}\right|_{u=u^\ast,v=v^\ast}=\left.\frac{\left.\partial \dot{v}\right|_{u=u^\ast}}{\partial v}\right|_{v=v^\ast} \\
	=&\left.\frac{\partial(\mathop{\E}_{x\sim \D_{real}}[\phi^\prime \di(x,\xt)\nabla_v
		\di(x)- \phi^\prime(1-\di(x))\nabla_v
		\di(x)])}{\partial v}\right|_{v=v^\ast}\\
	=& \left. \mathop{\E}_{x\sim \D_{real}}\phi^{\prime\prime}(\di(x)) \nabla_v\di(x)\nabla^T_v\di(x)+\phi^\prime(\di(x))\nabla^2_v\di(x)\right|_{v=v^\ast} \\
	& \left. \mathop{\E}_{x\sim \D_{real}}\phi^{\prime\prime}(\di(x,\xt)) \nabla_v\di(x)\nabla^T_v\di(x)-\phi^\prime(\di(x))\nabla^2_v\di(x)\right|_{v=v^\ast} \\
	=&2\phi^{\prime\prime}(\frac{1}{2}) \left. \mathop{\E}_{x\sim \D_{real}} \nabla_v\di(x)\nabla^T_v\di(x)\right|_{v=v^\ast}.
	\end{align}
	
	The matrix $\J_{vu}$ is:
	\begin{align}
	\J_{vu} =& \left.\frac{\partial \dot{v}}{\partial u}\right|_{u=u^\ast,v=v^\ast}=\left.\frac{\left.\partial \dot{v}\right|_{v=v^\ast}}{\partial u}\right|_{u=u^\ast} \\
	=& \frac{\partial}{\partial u}\mathop{\E}_{x\sim \D_{real}}\left.\mathop{\E}_{y\sim p_u(y)}[-\phi^\prime (\frac{1}{2})\nabla_v\di(y)]\right|_{u=u^\ast,v=v^\ast}\\
	=&\left. -\phi^\prime (\frac{1}{2})\mathop{\E}_{x\sim \D_{real}}\int_{\mathcal{Y}}\nabla_v\di(y)\nabla_u^Tp_u(y)dy\right|_{u=u^\ast,v=v^\ast}
	\end{align}
	
	The matrix 	$\J_{uv}$ is:
	\begin{align}
	\J_{uv}=& \left.\frac{\partial \dot{u}}{\partial v}\right|_{u=u^\ast,v=v^\ast}=\left.\frac{\left.\partial \dot{u}\right|_{u=u^\ast}}{\partial v}\right|_{v=v^\ast} \\
	=&- \frac{\partial}{\partial v}\mathop{\E}_{x\sim \D_{real}}\left.\int_{\mathcal{Y}}\nabla_up_u(y)[\phi(1-\di(x))]dy\right|_{u=u^\ast,v=v^\ast}\\
	=&- \mathop{\E}_{x\sim \D_{real}}\left.\int_{\mathcal{Y}}\nabla_up_u(y)(-\phi^\prime(1-\di(x)))\nabla_v^T\di(x)dy\right|_{u=u^\ast,v=v^\ast}\\
	=& \left. \phi^\prime (\frac{1}{2})\mathop{\E}_{x\sim \D_{real}}\int_{\mathcal{Y}}\nabla_up_u(y)\nabla_v^T\di(x)dy\right|_{u=u^\ast,v=v^\ast} =-\J_{vu}
	\end{align}
	
	Now, to show that $\J_{uu}$ is zero, we take any vector $\vb$ that is a perturbation in the generator space
	and show that $\vb^T\J_{uu}=0$. Here, we will use the limit definition of the derivative along a particular
	direction $\vb$.
	\begin{align}
	\left.\vb^T\frac{\partial \dot{u}}{\partial u}\right|_{u=u^\ast,v=v^\ast} =&\left.\vb^T\frac{\left. \partial \dot{u}\right|{v=v^\ast}}{\partial u}\right|_{u=u^\ast}\\
	=&-\lim\limits_{\substack{u-u^\ast=\epsilon\vb,\\\epsilon\rightarrow 0} } \frac{\E_{\D_{real}}\int_{\mathcal{Y}}(\nabla_u p_u(y) \phi(1-D_{v\ast}(y))+\lambda \|\nabla_\M (\psi(y)-\psi(x))\|^2)dy}{\epsilon}\\
	=& -\lim\limits_{\substack{u-u^\ast=\epsilon\vb,\\\epsilon\rightarrow 0} } \frac{\int_{supp(\D_{real})}(\nabla_u p_u(y) \phi(1-D_{v\ast}(x))dy}{\epsilon}\\
	=&-\phi(\frac{1}{2})\lim\limits_{\substack{u-u^\ast=\epsilon\vb,\\\epsilon\rightarrow 0} } \frac{\nabla_u\int_{supp(\D_{real})} p_u(y) dy}{\epsilon} =0\\
	\end{align}
	
	According to Lemma C.3 and Lemma G.2 in \cite{arora2017generalization}, the Jacobian matrix is Hurwitz and the training is locally exponentially stable.
\end{proof}

\end{document}